\documentclass[acmsmall]{acmart}
\AtBeginDocument{%
  }

\setcopyright{acmlicensed}
\copyrightyear{2018}
\acmYear{2018}
\acmDOI{XXXXXXX.XXXXXXX}

\setcopyright{cc}
\setcctype{by}
\acmJournal{PACMPL}
\acmYear{2025} \acmVolume{9} \acmNumber{OOPSLA2} \acmArticle{329} \acmMonth{10} \acmDOI{10.1145/3763107}

\usepackage[english]{babel}

\usepackage[T1]{fontenc}

\usepackage{graphicx}
\usepackage{microtype}
\usepackage{booktabs}

\usepackage{float}
\usepackage{wrapfig}
\usepackage[normalem]{ulem}

\usepackage{bm}
\usepackage{textcomp}
\usepackage{amsmath}
\usepackage{amsfonts}
\usepackage{mathtools}
\usepackage{pifont}

\usepackage{algorithm}
\usepackage{algpseudocode}

\usepackage{gensymb}

\usepackage{textcomp}
\usepackage{multirow}
\usepackage{lscape}
\usepackage{subcaption}

\usepackage{mdframed}
\usepackage{hyperref}
\usepackage{textgreek}
\usepackage{tabularx, colortbl,xcolor}
\definecolor{mycyan}{HTML}{005397}
\definecolor{myred}{HTML}{E13333}
\definecolor{mymagenta}{HTML}{BF3E87}
\definecolor{mypurple}{HTML}{1B2278}

\definecolor{tearose}{HTML}{F584C5}
\definecolor{coral}{HTML}{F67088}
\definecolor{dodger_blue}{HTML}{3BA3EC}
\definecolor{domino}{HTML}{BC9F48}
\definecolor{domino}{HTML}{BC9F48}
\definecolor{catalina_blue}{HTML}{1C3168}
\definecolor{catalina_blue}{HTML}{1C3168}
\definecolor{dark_scarlet}{HTML}{C63D52}
\definecolor{cerulean}{HTML}{0192A8}

\definecolor{tussock}{HTML}{C99E31}
\definecolor{p13}{HTML}{BFB5D7}
\definecolor{b14}{HTML}{BEA1A5}
\definecolor{y15}{HTML}{F0Cf61}

\definecolor{Merino}{HTML}{F3EEE3}
\definecolor{Twilight}{HTML}{4E518B}

\definecolor{light_silver}{HTML}{D9D9D9}

\definecolor{prussian_blue}{HTML}{023047}
\definecolor{pacific_blue}{HTML}{219EBC}
\definecolor{cornflower}{HTML}{90C9E6}

\newcolumntype{a}{>{\columncolor{p13}}l}

\usepackage[capitalize,noabbrev]{cleveref}
\crefname{ineq}{Inequality}{Inequalities}
\creflabelformat{ineq}{#2{\upshape(#1)}#3}

\theoremstyle{plain}
\newtheorem{theorem}{Theorem}[section]
\newtheorem{proposition}[theorem]{Proposition}
\newtheorem{lemma}[theorem]{Lemma}

\theoremstyle{definition}
\newtheorem{definition}[theorem]{Definition}

\theoremstyle{remark}
\newtheorem{remark}[theorem]{Remark}

\theoremstyle{remark}

\theoremstyle{remark}

\usepackage{braket}

\newcommand{\abs}[1]{\left\lvert#1\right\rvert}

\usepackage{stmaryrd}

\usepackage{tikz}
\usetikzlibrary{trees}
\usepackage{tikz-dependency}
\usepackage{tikz-3dplot}
\usetikzlibrary{chains,scopes}
\usetikzlibrary{arrows.meta,automata,positioning}
\usetikzlibrary{shapes.geometric, arrows}
\usepackage{pgfplots}

\definecolor{mypink1}{RGB}{255, 216, 176}
\definecolor{myblack1}{RGB}{0, 0, 0}
\definecolor{myalmostblack}{RGB}{51, 51, 51}
\definecolor{mygreen1}{RGB}{179, 254, 174}
\definecolor{myBlue2}{RGB}{148, 195, 225}
\definecolor{myBlue3}{RGB}{174, 219, 137}
\definecolor{myred1}{RGB}{255, 179, 178}
\definecolor{myblue1}{RGB}{162, 177, 195}

\tikzstyle{arrow} = [thick, draw=myblue1, line width=2, ->, >=stealth]

\pgfplotsset{compat=1.18}
\pgfplotsset{
  every axis/.append style = {thick},
  tick style = {thick,black},
  /tikz/normal shift/.code 2 args = {%
    \pgftransformshift{%
        \pgfpointscale{#2}{\pgfplotspointouternormalvectorofticklabelaxis{#1}}%
    }%
  },%
  shift/.style = {
    tick align        = outside,
    scaled ticks      = false,
    enlargelimits     = false,
    ticklabel shift   = {#1},
    axis lines*       = left,
    xtick style       = {normal shift={x}{#1}},
    ytick style       = {normal shift={y}{#1}},
    x axis line style = {normal shift={x}{#1}},
    y axis line style = {normal shift={y}{#1}},
  },
  shift/.default = 10pt,
  shift3d/.style = {
    shift=#1,
    ztick style       = {normal shift={z}{#1}},
    z axis line style = {normal shift={z}{#1}},
  },
  shift3d/.default = 10pt,
}

\usepackage{listings}
\usepackage{array}
\newcolumntype{H}{>{\setbox0=\hbox\bgroup}c<{\egroup}@{}}

\usepackage[utf8]{inputenc}

\begin{document}

\title{Correct-By-Construction: Certified Individual Fairness through Neural Network Training}

\author{Ruihan Zhang}
\email{rhzhang.2021@phdcs.smu.edu.sg}
\orcid{0009-0006-6669-7076}
\affiliation{%
  \institution{Singapore Management University}
  \city{Singapore}
  \country{Singapore}
}

\author{Jun Sun}
\affiliation{%
  \institution{Singapore Management University}
  \city{Singapore}
  \country{Singapore}}
\email{junsun@smu.edu.sg}
\orcid{0000-0002-3545-1392}







\renewcommand{\shortauthors}{Zhang and Sun}

\begin{abstract}
Fairness in machine learning is more important than ever as ethical concerns continue to grow. Individual fairness demands that individuals differing only in sensitive attributes receive the same outcomes. However, commonly used machine learning algorithms often fail to achieve such fairness. To improve individual fairness, various training methods have been developed, such as incorporating fairness constraints as optimisation objectives. While these methods have demonstrated empirical effectiveness, they lack formal guarantees of fairness. Existing approaches that aim to provide fairness guarantees primarily rely on verification techniques, which can sometimes fail to produce definitive results. Moreover, verification alone does not actively enhance individual fairness during training. To address this limitation, we propose a novel framework that formally guarantees individual fairness throughout training. Our approach consists of two parts, \emph{i.e.}, (1) provably fair initialisation that ensures the model starts in a fair state, and (2) a fairness-preserving training algorithm that maintains fairness as the model learns. A key element of our method is the use of randomised response mechanisms, which protect sensitive attributes while maintaining fairness guarantees. We formally prove that this mechanism sustains individual fairness throughout the training process. Experimental evaluations confirm that our approach is effective, \emph{i.e.}, producing models that are empirically fair and accurate. Furthermore, our approach is much more efficient than the alternative approach based on certified training (which requires neural network verification during training). 
\end{abstract}



\received{2025-03-26}
\received[accepted]{2025-08-12}

\maketitle

\section{Introduction}
Fairness in machine learning has become a vital concern as machine learning models increasingly influence decisions in critical domains such as hiring, credit scoring, and law enforcement~\cite{dwork2012fairness,caton2024fairness,locatello2019fairness,mhasawade2021machine}. These applications directly affect individuals' lives, making it crucial to ensure that these systems function without bias or discrimination.

Nevertheless, many existing machine-learning models in use do not meet this demand~\cite{barocas2023fairness}. Conventional training prioritises accuracy as an objective, inadvertently leading to models that optimise performance at the cost of fairness~\cite{kusner2017counterfactual}. Another major cause of unfair classification outcomes is the historical discrimination embedded in datasets. The existing unfairness may further aggravate during model training and eventually get amplified~\cite{brown2023detecting}. Hence, improving fairness within machine learning problems is an urgent and beneficial~\cite{bender2021dangers} mission.

Two prominent categories of fairness are individual fairness~\cite{zhang2020white,dwork2012fairness} and group fairness~\cite{jovanovic2023fare}. This work focuses on individual fairness, requiring all similarly situated individuals to receive similar treatment regardless of sensitive attributes such as race, gender, or age~\cite{dwork2012fairness}. Individual fairness is arguably one of the simplest and most intuitive definitions of fairness, as it guarantees fair outcomes on a personal level.

Enforcing individual fairness is highly challenging. One difficulty is that determining if this individual fairness holds requires traversing the entire input space. Since individual fairness holds only if every individual is fairly treated, the entire input space needs to be explicitly visited, which is difficult~\cite{kim2025fairquant}. An additional difficulty arises from the complex correlation between sensitive features and non-sensitive features~\cite{zemel2013learning}, which makes it hard to disentangle sensitive features from non-sensitive ones.

To improve individual fairness in machine learning, existing approaches can be broadly categorised into fairness verification and fairness-aware training. Verification-based methods, such as those proposed by Biswas et al.~\cite{biswas2023fairify} and Kim et al.~\cite{kim2025fairquant}, aim to detect fairness violations after a model has been trained by systematically analysing its behaviour across the input space. While these methods provide a structured framework for identifying unfair outcomes, they are computationally expensive and struggle to scale to complex neural networks, thereby limiting their practicality in real-world applications~\cite{biswas2023fairify}. More importantly, these methods do not improve the fairness of the trained model itself. 

On the other hand, fairness-aware training actively optimises individual fairness during training. Such training often represents individual fairness as an additional optimisation objective. These methods are typically considered to directly mitigate bias during model development and thus enhance fairness. However, such enhancement in individual fairness hardly comes with a formal guarantee~\cite{barocas2023fairness}. Biases can be observed to decrease in specific scenarios, but a provable extrapolation across all possible inputs is a separate issue~\cite{li2023accurate}. To the best of our knowledge, the work by Ruoss et al.~\cite{ruoss2020learning} remains the only approach that aims to provide individual fairness guarantees during training. Their method utilises an autoencoder to generate a fair feature representation, followed by certified training to ensure robustness. However, this approach may still fail to account for certain inputs~\cite{kim2025fairquant} and is constrained by high computational complexity and neural architecture specificity. Therefore, achieving a sound guarantee for individual fairness (certified individual fairness) remains an open challenge.

In this work, we tackle this problem through a correct-by-construction approach, \emph{i.e.}, by designing a training method that formally guarantees individual fairness across all inputs. The training has two parts: (1) provably fair initialisation and (2) fairness-preserving learning. Specifically, the individual fairness of the initial model is verified by fairness verifiers like Fairfy~\cite{biswas2023fairify}. We show that while verifying the individual fairness of a highly accurate model can be challenging~\cite{biswas2023fairify}, it is relatively straightforward for models where accuracy is not yet a concern. To ensure fairness throughout the training process, our fairness-preserving update mechanism leverages randomised response~\cite{warner1965randomized} techniques to safeguard sensitive attributes. This update mechanism is formally proven for correctness through partial derivatives. By structuring the training process to inherently maintain fairness, our approach guarantees individual fairness in the final model. Unlike existing fairness-aware training methods~\cite{li2023certifying}, where individual fairness is gradually improved, our method integrates fairness as an intrinsic property of the neural network from the start. As a result, post hoc verification for individual fairness in the trained model is no longer required.

Besides the theoretical individual fairness guarantee, we empirically evaluate the utility and efficiency of the proposed approach on multiple widely used benchmark datasets, and compare the results with those of multiple baseline methods. The experimental results confirm that our framework achieves high accuracy and computational efficiency with deep neural networks, demonstrating its potential as a scalable and reliable solution for fair machine learning.

Our main contributions are as follows. First, we theoretically establish that fairness can be guaranteed at the initialisation stage and preserved throughout the training process using carefully constructed parameter updates. Second, we propose a novel scalable training algorithm that integrates randomised response mechanisms to protect sensitive attributes while ensuring fairness propagation, and show that it can be effectively and efficiently applied in practice.

The remainder of this paper is structured as follows: \cref{sec:pre} introduces the preliminaries, providing the foundational concepts for our approach. \cref{sec:method} details our proposed certified training method. In \cref{sec:exp}, we describe the experimental setup and present the results. \cref{sec:related} reviews related work, and finally, \cref{sec:conclusion} concludes.

\section{Preliminaries and Problem Definition}
\label{sec:pre}

In this section, we provide an overview of the relevant background and formally define the research problem. In the background, we first introduce the notation of neural networks. Then, we show the formal definition of individual fairness accordingly. Meanwhile, we discuss individual fairness verification and certified training.

\subsection{Neural Networks}
\label{sec:nn}

A neural network can be formally interpreted as a function $f\!\mid_{\boldsymbol{\theta}}: \mathbb{X}\times\mathbb{S}\to \mathbb{Y}$, where $f$ defines the structure of the network and $\boldsymbol{\theta}$ represents the parameters of the network. The input domain $\mathbb{X}\times\mathbb{S}$ is Cartesian product of non-sensitive features (any $\bm{x}\in\mathbb{X}$) and sensitive features (any $\bm{s}\in\mathbb{S}$). While an input (including sensitive features) is often denoted simply as $\bm{x}$, we separate it into $\bm{x}$ and $\bm{s}$ for clarity and to facilitate the study of fairness. The output space is denoted as $\mathbb{Y}$, usually $\{1, 2, \ldots, K\}$ in the context of a $K$-category classification task.

For clarity and simplicity, we consistently use $f, \boldsymbol{\theta}, \mathbb{X}, \mathbb{S}, \mathbb{Y}$, and $K$ throughout this manuscript for representing the above-mentioned quantities. For specific instances of variables, such as a particular neural network structure ($f_0$) or a sensitive feature ($\bm{s}_0\in\mathbb{S}$), we include a subscript without implying sequence, order, or indexing. When referencing values from container-like structures, such as vectors, we use parentheses for indexing. For example, $\boldsymbol{\theta}_{(i)}$ represents the parameters on the $i$-th layer in a multi-layered network. Further, $f\!\mid_{\boldsymbol{\theta}_0}$ represents the restriction of $f$ to $\boldsymbol{\theta} = \boldsymbol{\theta}_0$, \emph{i.e.}, $f\!\mid_{\boldsymbol{\theta} = \boldsymbol{\theta}_0}(\bm{x}, \bm{s}) = f(\bm{x}, \bm{s}, \boldsymbol{\theta}_0)$.

\subsection{Individual Fairness}

Individual fairness mandates that similar individuals receive similar treatment. In the context of individual fairness, sensitive features (\emph{e.g.}, gender or race) should not influence the predictions made by neural networks. Formally, it is defined as follows.

\begin{definition}[Individual Fairness for a Sample~\cite{zhang2020white}]
\label{def:local}
Given a neural network $f\!\mid_{\boldsymbol{\theta}}$ and an input sample $(\bm{x}_1, \bm{s}_1) \in \mathbb{X}\times\mathbb{S}$, we say that $f\!\mid_{\boldsymbol{\theta}}$ is individually fair w.r.t $(\bm{x}_1, \bm{s}_1)$ if and only if $f\!\mid_{\boldsymbol{\theta}}(\bm{x}_1, \bm{s}_1) = f\!\mid_{\boldsymbol{\theta}}(\bm{x}_1, \bm{s}_2)$ for any $\bm{s}_2 \neq \bm{s}_1$. 
\end{definition}

While the above definition focuses on one sample, we can extend the definition and define individual fairness regarding multiple or all inputs in a certain domain.

\begin{definition}[Individual Fairness for a Domain\label{def:if}~\cite{biswas2023fairify,kim2025fairquant}]
Given a neural network $f\!\mid_{\boldsymbol{\theta}}$ and an input domain $\mathbb{X}\times\mathbb{S}$, we say that $f\!\mid_{\boldsymbol{\theta}}$ is individually fair for $\mathbb{X}\times\mathbb{S}$ if and only if the following condition is satisfied.
\begin{equation}
\label{eq:qualitative}
    \forall \bm{x} \in \mathbb{X}. \forall \bm{s}_1, \bm{s}_2\in\mathbb{S}. ~f\!\mid_{\boldsymbol{\theta}}(\bm{x}, \bm{s}_1) = f\!\mid_{\boldsymbol{\theta}}(\bm{x}, \bm{s}_2) 
\end{equation}
\end{definition}

\subsubsection{Individual fairness verification}
Getting to know if a neural network $f\!\mid_{\boldsymbol{\theta}}$ is individually fair over an input domain is highly non-trivial. The input space is often large or even infinite, preventing naive approaches such as explicitly enumerating all possible inputs. Hence, the verification of individual fairness for the input domain is typically addressed based on certain abstractions, such as symbolic interval analysis~\cite{wang2018formal,kim2025fairquant} and SMT solving~\cite{katz2019marabou,biswas2023fairify}.

\begin{figure}[t]
    \centering
    \tikzset{
    every neuron/.style={
        circle,
        fill=domino!10,
        draw,
        minimum size=30pt
      }
    }
    
    \tikzset{
    sensitive/.style={
        circle,
        draw,
        fill=red!20,
        dashed,
        minimum size=30pt
      }
    }

    \begin{subfigure}[t]{0.45\linewidth}
    \scriptsize
    \begin{tikzpicture}
        \pgfplotsset{
        width=0.3\textwidth,
        height=0.2\textwidth
        }

        \foreach \m [count=\y] in {1,2,3,4}
          \node [every neuron/.try, neuron \m/.try ] (input-\m) at (0,-1.2*\y) {$z_{i-1, \m}$};

        \node[sensitive, fill=blue!10] (i2) at (0, -2.4) {$z_{i-1, 2}$};
          
        \foreach \m [count=\y] in {1,2,3}
          \node [every neuron/.try, neuron \m/.try, fill=gray!5 ] (hidden-\m) at (2,-0.6-1.2*\y) {$z_{i, \m}$};

        \foreach \m [count=\y] in {1}
          \node [every neuron/.try, neuron \m/.try, fill=cerulean!4] (output-\m) at (4,-1.2-1.2*\y) {$z_{i+1, 1}$};
          \foreach \i in {1,...,4}
          \foreach \j in {1,...,3}
            \draw [->] (input-\i) -- (hidden-\j);
        
        \foreach \i in {1,...,3}
          \foreach \j in {1,...,1}
            \draw [->] (hidden-\i) -- (output-\j);
        
        \foreach \l [count=\x from 0] in {\textsc{Input}, \textsc{Hidden}, \textsc{Ouput}}
          \node [align=center, above] at (\x*2,-0.6) {\l };
    
    \end{tikzpicture}
    \end{subfigure}
    \quad\vline\quad
    \begin{subfigure}[t]{0.45\linewidth}
    \scriptsize
    \begin{tikzpicture}
        \pgfplotsset{
        width=0.3\textwidth,
        height=0.2\textwidth
        }

        \foreach \m [count=\y] in {1,2,3,4}
          \node [every neuron/.try, neuron \m/.try ] (input-\m) at (0,-1.2*\y) {$z_{i-1, \m}$};

        \node[sensitive] (i2) at (0, -2.4) {$z_{i-1, 2}$};
          
        \foreach \m [count=\y] in {1,2,3}
          \node [every neuron/.try, neuron \m/.try, fill=gray!5 ] (hidden-\m) at (2,-0.6-1.2*\y) {$z_{i, \m}$};

        \foreach \m [count=\y] in {1}
          \node [every neuron/.try, neuron \m/.try, fill=cerulean!4] (output-\m) at (4,-1.2-1.2*\y) {$z_{i+1, 1}$};
        
        \foreach \i in {1,...,4}
          \foreach \j in {1,...,3}
            \draw [->] (input-\i) -- (hidden-\j);
        
        \foreach \i in {1,...,3}
          \foreach \j in {1,...,1}
            \draw [->] (hidden-\i) -- (output-\j);
        
        \foreach \l [count=\x from 0] in {\textsc{Input}, \textsc{Hidden}, \textsc{Ouput}}
          \node [align=center, above] at (\x*2,-0.6) {\l };
    
    \end{tikzpicture}
    \end{subfigure}

\caption{Individual fairness of an example neural classifier. The left and right figures depict the same classifier (\emph{i.e.}, same $f\!\mid_{\boldsymbol{\theta}}$) taking two inputs, where only non-sensitive features are the same. When studying a neuron in the network, we may denote its current layer index as $i$ (regardless of its absolute position), and thus its preceding neurons as $i-1$ and its subsequent neuron as $i+1$.}
\label{fig:network}
\end{figure}
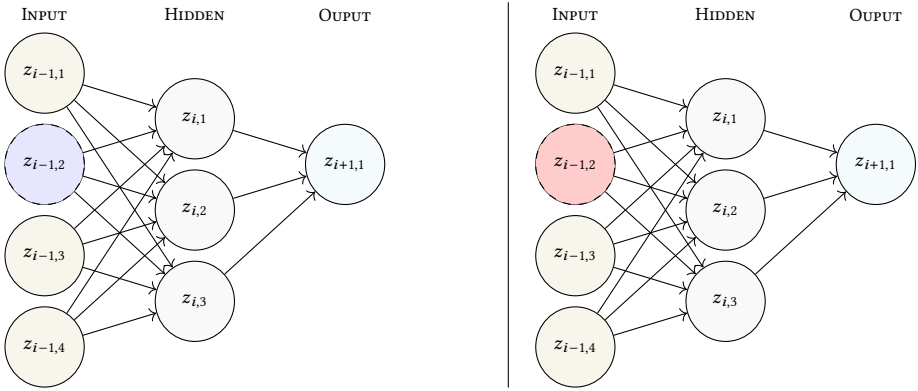

The individual fairness verification problem can be formulated as a general reachability problem that given the precondition on the input $(\bm{x}, \bm{s})$ and a classifier instance $f_0\!\mid_{\boldsymbol{\theta}_0}$, the output $f_0(\bm{x}, \bm{s}, \boldsymbol{\theta}_0)$ satisfies some postcondition~\cite{biswas2023fairify}. The precondition is typically established according to domain knowledge. Each (sensitive or non-sensitive) input feature is known to come from a certain range. For instance, if there is an income prediction task and a \textit{work-hours} (per week) feature, the precondition can be \textit{work-hours} $\in[1,100]$. Thereby, the verification problem is to find whether there is a sample violating the postcondition (\emph{i.e.}, the model generates different predictions when only some sensitive feature differs) within this input range. In recent studies~\cite{urban2020perfectly,zhang2020white,udeshi2018automated,albarghouthi2017fairsquare,bastani2019probabilistic,galhotra2017fairness}, verifying the satisfiability of preconditions and postconditions often involves constraint solving. That is, the verification problem is formulated as a constraint-solving problem and submitted to an SMT solver for satisfiability checking~\cite{biswas2023fairify}.

The verification process typically involves an input space abstraction method. Many abstraction methods have been proposed~\cite{pulina2010abstraction,elboher2020abstraction,kim2025fairquant}. The verification result can be provably fair (\emph{verified}), provably unfair (\emph{falsified}), or \emph{undecided} (\emph{e.g.}, due to timeout). 

\begin{remark}
Suppose $F(f\!\mid_{\boldsymbol{\theta}}, \mathbb{X}, \mathbb{S})=1$ denotes that neural network $f\!\mid_{\boldsymbol{\theta}}$ fulfils \cref{def:if}, \emph{i.e.}, $f\!\mid_{\boldsymbol{\theta}}$ is individually fair. Let $M$ be a sound verification process and $M(f\!\mid_{\boldsymbol{\theta}}, \mathbb{X}, \mathbb{S})=1$ means that the neural network can be verified. Then, we have $(M(f\!\mid_{\boldsymbol{\theta}}, \mathbb{X}, \mathbb{S})=1) \to (F(f\!\mid_{\boldsymbol{\theta}}, \mathbb{X}, \mathbb{S}) = 1)$, which means that $f\!\mid_{\boldsymbol{\theta}}$ is provably fair if it can be verified. This would be a sufficient condition but not a necessary condition.
\end{remark}

An intriguing observation is that many classifiers tend to be falsified/undecided in the verification process. As reported by Biswas et al. \cite{biswas2023fairify}, altogether 25 classifiers on three classification tasks have been collected. From verification, 19 of them are falsified, 6 are undecided, and none is verified. Furthermore, while verification does provide formal guarantees, it does not solve the individual fairness satisfiability problem because it does not improve the fairness of the model itself.

\subsubsection{Local or global individual fairness}
\begin{figure}[t]

    \centering
    \begin{subfigure}[t]{0.45\linewidth}
    \begin{tikzpicture}
    \begin{axis}[
        width=\linewidth, height=0.75\linewidth,
        xmin=-2.4, xmax=3.3,
        ymin=-1.5, ymax=2.3,
        axis lines=box,              
        tick align=outside,
        xtick pos=bottom,            
        ytick pos=left,              
        tick style={black, thin},   
        major tick length=2pt,       
        xtick distance=1,
        enlargelimits=false,
        clip=false,
        legend columns=2,
        legend style={
            at={(1.2,1.2)}, anchor=north,
            draw=none, rounded corners, fill=none,
            cells={anchor=west},font=\scriptsize
        }
    ]

    \addplot+[only marks, draw=blue!60!black, mark=*, mark size=0.5pt] table {points.txt};
    \addlegendentry{Provided Input$\quad$}
    
    \addplot+[only marks, mark=square, mark size=2.2pt,
              mark options={draw=red!85!black, line width=1.1pt, fill=none}]
              table {points.txt};
    \addlegendentry{Certified Domain}
    
    \end{axis}
    \end{tikzpicture}
    \caption{Local individual fairness\label{fig:localif}}
    
    \end{subfigure}
    \begin{subfigure}[t]{0.45\linewidth}
    \begin{tikzpicture}
    \begin{axis}[
        width=\linewidth,
        height=0.75\linewidth,
        xmin=-2.4, xmax=3.3,
        ymin=-1.5, ymax=2.3,
        axis lines=box,              
        tick align=outside,
        xtick pos=bottom,            
        ytick pos=left,              
        tick style={black, },   
        major tick length=2pt,       
        xtick distance=1,
        enlargelimits=false,
        clip=false,
    ]
    
    \addplot[only marks, draw=blue!60!black, mark=*, mark size=0.5pt]
    table{points.txt};

    \draw[thick, red!85!black]
      (axis cs:-2.1,-1.35) rectangle (axis cs:3.1,2.15);
    
    \end{axis}
    \end{tikzpicture}
    \caption{Global individual fairness\label{fig:globalif}}
    \end{subfigure}
    \caption{Local (left) or global (right) individual fairness. Global individual fairness always covers the entire input space. Local individual fairness often covers the regions around given inputs. The left figure is one of many types of local individual fairness, depending on the subset choice. }
    \label{fig:localglobal}
\end{figure}
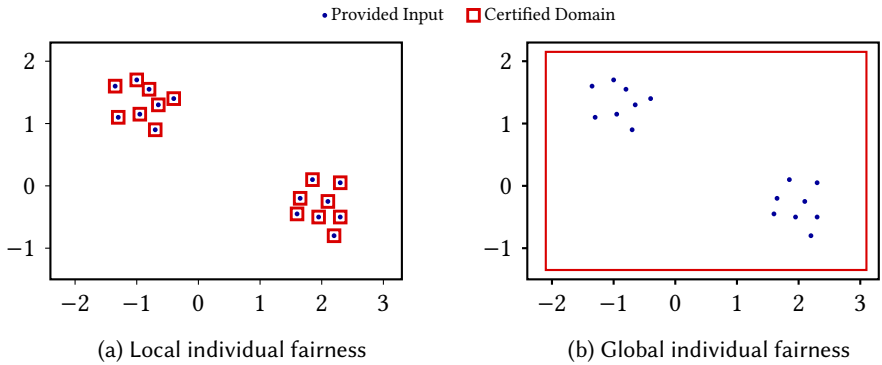

The fairness described in \cref{def:local} can be seen as a type of local individual fairness, which corresponds to the neural network property for a single input. There may also be other variations of local individual fairness, such as fairness within a subset of the input space surrounding this input~\cite{ruoss2020learning}. In contrast, the fairness in \cref{def:if} can be seen as global individual fairness, where the entire input space needs to be considered. Throughout this manuscript, fairness refers to \emph{global} individual fairness unless stated otherwise. \cref{fig:localglobal} illustrates local and global individual fairness. Theoretically, if a neural network satisfies global individual fairness, it also satisfies local individual fairness, but the reverse does not necessarily hold. Moreover, verifying global individual fairness is more challenging than verifying local individual fairness because it requires checking the entire space~\cite{kim2025fairquant}.

\subsubsection{Certified training}

Certified training aims to tackle the challenge that training can hardly be guaranteed, and verification alone does not actively make improvements. This type of approach has been extensively studied for neural network perturbation robustness~\cite{zhang2024certified}. Robustness verification checks whether a classifier $f\!\mid_{\boldsymbol{\theta}}$ maintains consistent predictions under small perturbations. Specifically, given an input $\bm{x}\in\mathbb{X}$ and a small constant $\epsilon$, $f\!\mid_{\boldsymbol{\theta}}$ is considered robust against \( \epsilon \)-perturbations if and only if $f\!\mid_{\boldsymbol{\theta}}(\bm{x}) = f\!\mid_{\boldsymbol{\theta}}(\bm{x}')$ for all $\bm{x}' \in \mathbb{X}$ such that $\|\bm{x} - \bm{x}'\| \leq \epsilon$. Robustness verification methods can be used to optimise classifiers during training~\cite{li2023sok,mu2025get,zhang2023towards}, a process known as certified training (of robustness), which provides formal guarantees of improved robustness. 

Certified training is also explored for individual fairness guarantees. A certain level of individual fairness guarantee is obtained from the robustness verifiers~\cite{ruoss2020learning}. However, robustness verifiers only provide guarantees within a small $\epsilon$ perturbation around the non-sensitive features. This often leads to a local guarantee of individual fairness because forcing $\epsilon$ to be a large value rarely works in practice~\cite{biswas2023fairify,kim2025fairquant}.

\subsection{Problem Definition} 

We are looking for a solution to the individual fairness satisfiability problem. This requires a model to be trained with adequate utility and individual fairness. Most importantly, a formal guarantee of individual fairness should be provided.

This raises the following questions. First, is it possible to find training that enhances individual fairness and also provides a formal guarantee? If so, what training approach accomplishes this, and how does it establish the guarantee? With these questions in mind, we now formally define our problem as in \cref{def:problem}.

\begin{definition}[Neural Network Certifiable Fair Training Problem]
\label{def:problem}
We are given a neural network structure $f$, a set of training examples $\set{(\bm{x}_i, \bm{s}_i, y_i)}_{i=1}^n$ where $n$ is the sample size. For $i = 1, \ldots, n$, we have $ (\bm{x}_i, \bm{s}_i, y_i)\in \mathbb{X}\times\mathbb{S}\times\mathbb{Y}$. The certifiable fair training problem is to construct a neural network $f\!\mid_{\boldsymbol{\theta^*}}$ such that $f\!\mid_{\boldsymbol{\theta^*}}$ satisfies $F(f\!\mid_{\boldsymbol{\theta^*}}, \mathbb{X}, \mathbb{S})=1$.
\end{definition}

\section{Method}
\label{sec:method}

In the following, we first show that individual fairness can be established as an inherent property of neural networks through two key parts. (1) Fairness can be guaranteed at initialisation by constructing a provably fair initialisation for the network. (2) We identify structural properties that propagate fairness from specific sub-networks to the entire model, enabling fairness to be maintained as the network updates. After that, we introduce fairness-preserving training algorithms that leverage the presented properties to protect sensitive attributes while maintaining scalability and accuracy. A schematic overview of the proposed method is shown in \cref{fig:method}.

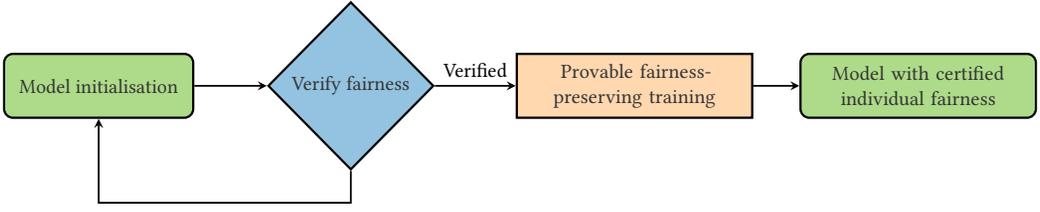
\begin{figure}[t]
    \centering
    \resizebox{\textwidth}{!}{ 
    \begin{tikzpicture}[node distance=1cm]
    
        \node (init) [rectangle, rounded corners, minimum width=3cm, minimum height=1cm, text centered, font=\small, color=myalmostblack, draw=myblack1, line width=1, fill=myBlue3, xshift=-2cm] {Model initialisation};
        \node (verify) [diamond, minimum width=2cm, minimum height=1cm, text centered, font=\small, color=myalmostblack, draw=myblack1, line width=1, fill=myBlue2, right of=init, xshift=3cm] {Verify fairness};
        \node (train) [rectangle, minimum width=2cm, minimum height=1cm, text centered, font=\small, color=myalmostblack, draw=myblack1, line width=1, fill=mypink1, text width=3.5cm, align=center, right of=verify, xshift=3.5cm] {Provable fairness- preserving training};
        \node (certified) [rectangle, rounded corners, minimum width=3cm, minimum height=1cm, text centered, font=\small, color=myalmostblack, draw=myblack1, line width=1, fill=myBlue3, right of=train, xshift=3.5cm, text width=3.5cm, align=center] {Model with certified individual fairness};
    
        \draw [thick, ->, >=stealth] (init) -- (verify);
        \draw [thick, ->, >=stealth] (verify) -- node[anchor=south, font=\small, align=center] {Verified} (train);
        \draw [thick, ->, >=stealth] (train) -- (certified);
    
        \draw [thick, ->, >=stealth, ] (verify.south) |- ++(0,-0.5) -| (init.south);

    \end{tikzpicture}
    }
    \caption{Our approach actively ensures fairness from the start, reliably achieving certified individual fairness.\label{fig:method}}    
\end{figure}

\subsection{Fairness Can Be Guaranteed at Neural Network Initialisation}

Fairness could be guaranteed from the start, instead of an afterthought when the neural network is trained. \cref{lemma:init} establishes the existence of an initialisation where the network is provably fair. By choosing parameter values that offer equality across sensitive attributes, the network ensures that the output remains identical regardless of the sensitive feature values for any input. For example, initialising all weights and other parameters to zeros guarantees fairness trivially as the model produces the same output for all inputs. Besides, random initialisations that satisfy $f\!\mid_{\boldsymbol{\theta}_{0}}(\bm{x},\bm{s})=0$ also exist and can be used as fair initialisations. Moreover, provably fair (but possibly not very accurate, if any) classifiers from the literature~\cite{biswas2023fairify,kim2025fairquant} could also be adopted to initialise the neural network.

\begin{lemma}
\label{lemma:init}
    There exists some parameter initialisation $\boldsymbol{\theta}_0$ for a neural network $f\!\mid_{\boldsymbol{\theta}} : \mathbb{X}\times \mathbb{S} \to \mathbb{Y}$ such that, for all $\bm{s}_1,\bm{s}_2 \in \mathbb{S}$ and for all $\bm{x}\in \mathbb{X}$,
    \begin{equation}
        f\!\mid_{\boldsymbol{\theta}_0}(\bm{x},\bm{s}_1)\;=\;    f\!\mid_{\boldsymbol{\theta}_0}(\bm{x},\bm{s}_2).        
    \end{equation}
    That is, $f\!\mid_{\boldsymbol{\theta}_0}$ is individually fair.
\end{lemma}

\begin{proof}
    To prove the existence of a provably fair classifier, it is sufficient to find at least one example that satisfies the condition. We can trivially do so by choosing all network weights and biases to produce a constant output $c$.  For instance, set every weight to $0$ and every bias to some constant $c$.  Then $\forall (\bm{x},\bm{s}),~f\!\mid_{\boldsymbol{\theta}_0}(\bm{x}, \bm{s})   = c$. Trivially, for every $\bm{s}_1,\bm{s}_2$, $f\!\mid_{\boldsymbol{\theta}_0}(\bm{x}, \bm{s}_1) = c = f\!\mid_{\boldsymbol{\theta}_0}(\bm{x}, \bm{s}_2)$, so we have perfect individual fairness.
    This completes the proof.   
\end{proof}

This result is significant for two reasons. First, it shifts fairness from being externally injected to a static internal property. Second, it provides a critical starting point for guaranteeing fairness during subsequent updates.

In the proof of \cref{lemma:init}, we use the trivial zero initialisation of parameters to show the existence of fair initialisation. Yet, such neural networks can hardly update properly. To this end, one solution is to adopt near-zero random initialisations. Their fairness is not guaranteed in general, but we can prove it using formal fairness verifiers~\cite{biswas2023fairify}. For instance, we can let each parameter take a random value from a Bernoulli trial with success probability $p$ (\emph{e.g.}, 0.5) and outcomes $-e^\phi$ and $+e^\phi$, where $\phi$ is some small power like $-10$. Then, as long as an initialisation instance is verified to be fair, it serves as proof of \cref{lemma:init} as well.

We note that such initialisations, while fair, often lack practical utility. This raises the next question, \emph{i.e.}, although we can start with fairness, can we keep it while improving the network utility? This question motivates the next steps, which focus on propagating fairness through the network structure and preserving it during training. \cref{lemma:init} thus serves as the foundation upon which a fair and functioning neural network can be built.

\subsection{Preserving Individual Fairness in Neural Networks}

We show that all neural networks follow two properties of individual fairness. First, if every neuron in the second-last layer (\emph{i.e.}, all neurons that input to the output neuron of this network) is individually fair, the entire neural network is provably individually fair (\cref{lemma:sub}). Second, if the parameters of an individually fair neural network are updated in a way such that a certain condition is satisfied, the updated network remains individually fair (\cref{lemma:either}).

To establish these properties, we extend the neural network formulation from \cref{sec:nn} to include the internal structure. A neural network can be formulated as a graph, where each node corresponds to a neuron. An example neural network is shown in \cref{fig:network}. Let $z$ represent a neuron at a layer such that $z_{n,1}$ (where $n$ denotes the last layer index of neural network) represents the output of this entire neural network, \emph{i.e.}, $z_{n,1} = f(\bm{x}, \bm{s}, \boldsymbol{\theta})$. We use fully connected networks in \cref{fig:network} to demonstrate this property, yet this notation indeed holds for all networks that can be represented as a directed acyclic graph, because we simply need the neurons at each layer to be written as local functions of neurons at their previous layer, \emph{i.e.},
\begin{equation}
\label{eq:decompose}
\forall z_{i,j}\in \set{z_{i,1}, z_{i,2}, \ldots}.\quad z_{i,j} = f_{(i,j)}\left(z_{i-1,1}, z_{i-1,2}, \ldots, \boldsymbol{\theta}_{(i,j)}\right),   
\end{equation}
where the overall function $f$ can be represented in a composition form $f(\bm{x}, \bm{s}, \boldsymbol{\theta}) \coloneqq f_{(n,1)}\circ f_{(n-1,1)}\circ\ldots\circ f_{(1,1)}$. Further, we can compose all functions in the preceding operations (from the first layer to $i$-th layer) of \textcolor{black}{local neural function} $f_{(i,j)}$ to get a \textcolor{black}{subnetwork} function $\tilde{f}_{(i,j)}$, \emph{i.e.},
\begin{equation}
\label{eq:subnet}
    \tilde{f}_{(i,j)}(\bm{x},\bm{s}, \boldsymbol{\theta}_{(1,1)}\ldots \boldsymbol{\theta}_{(i,j)}) \coloneqq f_{(i,j)}\circ f_{(i-1,1)}\circ\ldots,
\end{equation}
\textcolor{black}{where the output value of $\tilde{f}_{(i,j)}(\bm{x},\bm{s}, \boldsymbol{\theta}_{(1,1)}\ldots \boldsymbol{\theta}_{(i,j)})$ is $z_{(i, j)}$.} A ``concept of layer'' for the neuron under consideration is not necessary, meaning that the sibling relationship between a neuron and the neuron under consideration ($z_{i, \cdot}$) is only valid upon specifying a common subsequent neuron. When $z_{i, \cdot}$ is under consideration, any input neuron to $z_{i, \cdot}$ will also be denoted as $z_{i-1, 1}, z_{i-1, 2}$, etc.

One can see that \cref{eq:subnet} denotes the function of subnetworks. The largest sub-network of a network is itself and thus $\tilde{f}_{(n,1)}(\bm{x},\bm{s}, \boldsymbol{\theta}_{(1,1)}\ldots \boldsymbol{\theta}_{(n,1)}) = f(\bm{x}, \bm{s}, \boldsymbol{\theta}) $. With this notation, we can introduce the fairness composition property as follows.

\begin{lemma}
\label{lemma:sub}
    A neural network is individually fair if all its second-last sub-networks are individually fair. Formally, given a neural network $f\!\mid_{\boldsymbol{\theta}}$, then 
    \begin{equation}
    \begin{aligned}
    \left(\forall j.\forall \bm{x} \in \mathbb{X}. \forall \bm{s}_1, \bm{s}_2\in\mathbb{S}.~ \left.\tilde{f}_{(n-1,j)}\right|_{\ldots \boldsymbol{\theta}_{(n-1,j)}}(\bm{x}, \bm{s}_2)  = \left.\tilde{f}_{(n-1,j)}\right|_{\ldots \boldsymbol{\theta}_{(n-1,j)}}(\bm{x}, \bm{s}_2)\right)\\
    \to \left(\forall \bm{x} \in \mathbb{X}. \forall \bm{s}_1, \bm{s}_2\in\mathbb{S}. ~f\!\mid_{\boldsymbol{\theta}}(\bm{x}, \bm{s}_1) = f\!\mid_{\boldsymbol{\theta}}(\bm{x}, \bm{s}_2) \right)\qquad
    \end{aligned}
    \end{equation}
\end{lemma}

\begin{proof}
    We first express $f(\bm{x}, \bm{s}, \boldsymbol{\theta})$ in the following form, where the final layer depends on the sub-network outputs. 
    \begin{equation}
    \begin{aligned}
        f(\bm{x}, \bm{s}, \boldsymbol{\theta}) &= f_{(n,1)}\left(z_{n-1,1}, z_{n-1,2}, \ldots, \boldsymbol{\theta}_{(n,1)}\right)\\
        &=   f_{(n,1)}\left(\tilde{f}_{(n-1,1)}(\bm{x}, \bm{s} ,\ldots\boldsymbol{\theta}_{(n-1,1)}), \tilde{f}_{(n-1,2)}(\bm{x}, \bm{s} ,\ldots\boldsymbol{\theta}_{(n-1,2)})\ldots\right)   
    \end{aligned}
    \end{equation}
    Take any $\bm{x}\in \mathbb{X}$, and any $\bm{s}_1,\bm{s}_2\in \mathbb{S}$. By the hypothesis, for each neuron $j$ on the second-last layer, we can see $ \tilde{f}_{(n-1,j)}(\bm{x}, \bm{s}_1,\ldots\boldsymbol{\theta}_{(n-1,j)}) = \tilde{f}_{(n-1,j)}(\bm{x}, \bm{s}_2,\ldots\boldsymbol{\theta}_{(n-1,j)})$. Thus, the entire inputs to $f_{(n,1)}$ is the same for $\bm{s}_1$ and $\bm{s}_2$. Since the final layer $f_{(n,1)}$ does not depend on the input $\bm{s}$ and it receives the \emph{same} inputs in these two cases, and we can conclude
    \begin{equation}
    \begin{aligned}
        f_{(n,1)}(\tilde{f}_{(n-1,1)}(\bm{x}, \bm{s}_1 ,\ldots\boldsymbol{\theta}_{(n-1,1)}), \tilde{f}_{(n-1,2)}(\bm{x}, \bm{s}_1 ,\ldots\boldsymbol{\theta}_{(n-1,2)})\ldots) =\qquad\\
        f_{(n,1)}(\tilde{f}_{(n-1,1)}(\bm{x}, \bm{s}_2 ,\ldots\boldsymbol{\theta}_{(n-1,1)}), \tilde{f}_{(n-1,2)}(\bm{x}, \bm{s}_2 ,\ldots\boldsymbol{\theta}_{(n-1,2)})\ldots).        
    \end{aligned}
    \end{equation}
    Here, the left-hand side is exactly $f(\bm{x},\bm{s}_1)$ and the right-hand side is $f(\bm{x},\bm{s}_2)$, hence $f$ is individually fair, as required.

\end{proof}

This property can be viewed as a chain for individual fairness, where each subnetwork acts as a link. As long as every link is strong, the chain remains unbroken, guaranteeing individual fairness cascades throughout the network. This modular perspective provides a clear pathway to certifying individual fairness. However, fairness is not static. As the network updates during training, even a strong chain can be weakened if fairness is not carefully preserved. This raises the following question: How do we maintain fairness as the network learns?

The answer is to rely on a fairness updating property as we present below. The fairness updating property also relies on the network composition, but here the focus is gradient.  From \cref{eq:decompose}, each local neural function $f_{(i,j)}$ takes the neuron values of its previous layer and its parameters as input, and outputs its own neuron value $z_{(i,j)}$. Hence, the partial derivative of $f_{(i,j)}$ with respect to its parameters $\boldsymbol{\theta}_{(i,j)}$ would be denoted as $\frac{\partial z_{i,j}}{\partial \boldsymbol{\theta}_{(i,j)}}$. Therefore, an update to the neural network at time step $t-1$ to $t$ can be expressed as follows, where $(\cdot^{(t)})$ denotes time step.

\begin{equation}
\label{eq:update}
    z_{i,j}^{(t)}  = f_{(i,j)}\left(z_{i-1,1}^{(t)}, z_{i-1,2}^{(t)}, \ldots,\boldsymbol{\theta}_{(i,j)}^{(t-1)} - \eta \frac{d\text{loss}}{d z_{n,1}}\left.\frac{\partial z_{n,1}}{\partial \boldsymbol{\theta}_{(i,j)}}\right|_{z_{i,j} = z_{i,j}^{(t-1)}, \boldsymbol{\theta}_{(i,j)} = \boldsymbol{\theta}_{(i,j)}^{(t-1)}}\right)
\end{equation}
The following lemma states that a neural network can be updated without compromising its individual fairness.

\begin{lemma}
\label{lemma:either}
    Suppose a neural network $f\!\mid_{\boldsymbol{\theta}^{(t-1)}}$ is individually fair at time step $t-1$.  Then if the update to each parameter $\boldsymbol{\theta}_{(i,j)}$ from $\boldsymbol{\theta}^{(t-1)}$ to $\boldsymbol{\theta}^{(t)}$ satisfies at least one of the following conditions, then the neural network at time step $t$ is also individually fair. (1) For all sensitive features, the gradients (of this parameter) sum up to 0 (\cref{eq:zero}). (2) For all sensitive features, the gradients of this parameter are the same (\cref{eq:same}).
    \begin{equation}
    \label{eq:zero}
        \textcolor{black}{\forall \bm{x}_1 \in \mathbb{X}.~ \sum _{\bm{s}_1\in \mathbb{S}}}\left.\frac{\partial z_{n,1}}{\partial \boldsymbol{\theta}_{(i,j)}}\right|_{\bm{x}=\bm{x}_1, \bm{s}=\bm{s}_1, \boldsymbol{\theta}_{(i,j)} = \boldsymbol{\theta}_{(i,j)}^{(t-1)}} = 0
    \end{equation}
    \begin{equation}
    \label{eq:same}
        \forall \bm{x}_1\in\mathbb{X}, \bm{s}_1, \bm{s}_2\in\mathbb{S}.~ \left.\frac{\partial z_{n,1}}{\partial \boldsymbol{\theta}_{(i,j)}}\right|_{\bm{x}=\bm{x}_1, \bm{s}=\bm{s}_1, \boldsymbol{\theta}_{(i,j)} = \boldsymbol{\theta}_{(i,j)}^{(t-1)}} = \left.\frac{\partial z_{n,1}}{\partial \boldsymbol{\theta}_{(i,j)}}\right|_{\bm{x}=\bm{x}_1, \bm{s}=\bm{s}_2,\boldsymbol{\theta}_{(i,j)} = \boldsymbol{\theta}_{(i,j)}^{(t-1)}}  
    \end{equation}
\end{lemma}

\begin{proof}
    Let us define the difference in outputs for two sensitive inputs $\bm{s}_1,\bm{s}_2$ (with fixed nonsensitive input $\bm{x}$) as $ g(\bm{x}, \bm{s}_1,\bm{s}_2, \boldsymbol{\theta}) = f(\bm{x}, \bm{s}_1,\boldsymbol{\theta}) - f(\bm{x}, \bm{s}_2,\boldsymbol{\theta})$. Individual fairness at time $t-1$ means for all $\bm{x},\bm{s}_1,\bm{s}_2$, $g(\bm{x},\bm{s}_1,\bm{s}_2,\boldsymbol{\theta}^{(t-1)}) = 0$. We show that $g(\bm{x},\bm{s}_1,\bm{s}_2,\boldsymbol{\theta}^{(t)}) = 0$ for all $\bm{x},\bm{s}_1,\bm{s}_2$ after the update, if each parameter $\boldsymbol{\theta}_{(i,j)}$ in $\boldsymbol{\theta}$ satisfies at least one of the two conditions.

    The update $\Delta \boldsymbol{\theta}_{(i,j)}$ is proportional to the sum of the partial derivatives. Hence, if aggregating across all $\bm{s}$ yields zero, no net change favours one $\bm{s}$ over another. Then, by function restriction, we let $\boldsymbol{\theta}_\alpha$ represent parameters that does not satisfy this condition and $f(\bm{x}, \bm{s},\boldsymbol{\theta}) = f\!\mid_{\boldsymbol{\theta}\setminus\boldsymbol{\theta}_\alpha}(\bm{x}, \bm{s}, \boldsymbol{\theta}_{\alpha})$. Therefore, we get
    \begin{equation}
    \begin{aligned}
        f\!\mid_{\boldsymbol{(\theta}\setminus\boldsymbol{\theta}_\alpha)^{(t)}}&(\bm{x}, \bm{s}_1, \boldsymbol{\theta}_{\alpha}^{(t-1)}) = f\!\mid_{\boldsymbol{(\theta}\setminus\boldsymbol{\theta}_\alpha)^{(t-1)}}(\bm{x}, \bm{s}_1, \boldsymbol{\theta}_{\alpha}^{(t-1)})\\
        &= f\!\mid_{\boldsymbol{(\theta}\setminus\boldsymbol{\theta}_\alpha)^{(t-1)}}(\bm{x}, \bm{s}_2, \boldsymbol{\theta}_{\alpha}^{(t-1)}) = f\!\mid_{\boldsymbol{(\theta}\setminus\boldsymbol{\theta}_\alpha)^{(t)}}(\bm{x}, \bm{s}_2, \boldsymbol{\theta}_{\alpha}^{(t-1)}).
    \end{aligned}
    \end{equation}
    
    Further, suppose gradients of $\boldsymbol{\theta}_\alpha$ does not depend on $\bm{s}$. Then, for all $\bm{x}_1\in\mathbb{X}$, for all possible $\bm{s}_1,\bm{s}_2$, $\nabla_{\boldsymbol{\theta}_{\alpha}} (f\!\mid_{\boldsymbol{\theta}\setminus\boldsymbol{\theta}_\alpha}(\bm{x}_1, \bm{s}_1, \boldsymbol{\theta}_{\alpha})) = \nabla_{\boldsymbol{\theta}_{\alpha}} (f\!\mid_{\boldsymbol{\theta}\setminus\boldsymbol{\theta}_\alpha}(\bm{x}_1, \bm{s}_2, \boldsymbol{\theta}_{\alpha}))$ at time $t-1$, and thus $\nabla_{\boldsymbol{\theta}_{\alpha}} g(\bm{x}_1, \bm{s}_1,\bm{s}_2, \boldsymbol{\theta}) = 0$. For small update $\eta$, the best linear approximation near $\boldsymbol{\theta}^{(t-1)}$ would be  
    \begin{equation}
    \begin{aligned}
        &g\!\mid_{\boldsymbol{\theta}\setminus\boldsymbol{\theta}_\alpha}\left(\bm{x}_1, \bm{s}_1,\bm{s}_2, \boldsymbol{\theta}^{(t)}_\alpha\right) \\
        &= g\!\mid_{\boldsymbol{\theta}\setminus\boldsymbol{\theta}_\alpha}\left(\bm{x}_1, \bm{s}_1,\bm{s}_2, \boldsymbol{\theta}^{(t-1)}_\alpha\right) + \nabla_{\boldsymbol{\theta}_{\alpha}} g(\bm{x}_1, \bm{s}_1,\bm{s}_2, \boldsymbol{\theta})\left|_{\boldsymbol{\theta}^{(t-1)}_\alpha}~\Delta\boldsymbol{\theta}_{\alpha}\right. = 0.
    \end{aligned}
    \end{equation}
    Therefore, the network at time $t$ is also individually fair.
\end{proof}

Intuitively, individual fairness is preserved if either the gradients of a parameter, aggregated across all sensitive features, sum to zero, or if the gradients are identical for all sensitive features. These conditions ensure that updates do not introduce bias, therefore maintaining fairness as the network learns.

Together, \cref{lemma:sub} and \cref{lemma:either} establish a framework where fairness propagates within the network structure and is constantly preserved during training. This paves the way for developing a certified fair training algorithm as we present next.

\subsection{One-step Fair Update With Randomised Response}

In the following, we discuss in detail how we design the training algorithm so that it can effectively leverage the properties introduced above. Specifically, we add a randomised response mechanism to the sensitive feature during training. The proposed training is described in \cref{alg:gd}.  Essentially, we use a randomised response mechanism (also represented in \cref{alg:gd} Lines 6 - 12) to protect the sensitive feature. We let $\omega_0, \omega_1, \omega_2, \ldots$ denote the sensitive feature values, and the released sensitive feature follows the probability presented in \cref{eq:rrs_plus}, where $\mathbf{1}_{\text{condition}}$ returns 1 if the condition holds, and 0 otherwise.

\begin{equation}
    \label{eq:rrs_plus}
    \forall i \in \set{\omega_0, \omega_1, \omega_2, \ldots, \omega_{\abs{\mathbb{S}}-1}}.\quad P(\textnormal{s} =\omega_j| \bm{s}=\omega_i) = 
        \frac{e^{\gamma\cdot\mathbf{1}_{i=j}}}{e^\gamma + \abs{\mathbb{S}}-1}
\end{equation}
For example, in a common (and simplified) binary setting, as done by \citet{kim2025fairquant}, sensitive categorical values are $\omega_0 = -1$ and $\omega_1 = +1$. In such a case, \cref{eq:rrs_plus} can be simplified as \cref{eq:rrs} in a binary case, where we just substitute $\abs{\mathbb{S}}=2$.
\begin{equation}
    \label{eq:rrs}
    \forall i \in \set{-1,1}.\quad P(\textnormal{s} =j| \bm{s}=i) = \begin{cases}
        \frac{e^\gamma}{e^\gamma + 1}, & j = i,\\
        \frac{1}{e^\gamma + 1}, & j = -i .
    \end{cases}
\end{equation}
In some other cases, $\omega_0, \omega_1, \omega_2, \ldots$ may represent multiple sensitive features, \emph{e.g.}, the UCI Credit dataset has a feature ``Personal status and sex'', and $\omega_0, \omega_1, \omega_2, \ldots$ could represent ``male, single'', ``female, divorced/separated/married'', etc. Thus, a simple way to handle multiple sensitive features simultaneously would be to use a single composite feature~\cite{benussi2022individual}.

\begin{algorithm}[t]
\caption{Gradient Descent with Randomised Responded Sensitive Features\label{alg:gd}}
\begin{algorithmic}[1] 
\Require $\eta$ (learning rate), $\bm{\theta}_0$ (parameters), $\set{(\bm{x}_i, \bm{s}_i, y_i)}_{i=1}^n$ (training dataset), $ \ell(\bm{\theta}, \bm{x}, \bm{s}, y)$ (objective/loss function), $\gamma$ (privacy budget)    
\Ensure $\bm{\theta}_t$
\State $p \gets e^\gamma / (e^\gamma + |\mathbb{S}| - 1)$
\For{$t = 1$ to $\ldots$}
    \State $\bm{g}_t \gets 0$
    \For{$i = 1,\ldots,$}
        \For{$j = 1, \ldots, n$}
            \State Sample $q \sim U(0, 1)$
            \If{$q > p$}
                \State $\hat{\bm{s}} \sim U(\mathbb{S} \setminus \{\bm{s}_j\})$
            \Else
                \State $\hat{\bm{s}} \gets \bm{s}_j$
            \EndIf
            \State $\bm{g}_t \gets \bm{g}_t + \nabla_{\bm{\theta}} \ell(\bm{\theta}_{t-1}, \bm{x}_i, \hat{\bm{s}}, y_i)/n$
        \EndFor
    \EndFor
    \State $\bm{\theta}_t \gets \bm{\theta}_{t-1} - \eta \bm{g}_t/i$
\EndFor
\end{algorithmic}
\end{algorithm}

The subsequent question is to determine $\gamma$ such that individual fairness can be maintained during updating. We find that if it makes sure that the gradients for each parameter of an individually fair neuron sum up to zero, then the update is expected to preserve fairness. This is formally captured in \cref{thm:update}.

\begin{theorem}
\label{thm:update}
    Given an individually fair neural network $f\!\mid_{\boldsymbol{\theta}  = \boldsymbol{\theta}^{(t-1)}}$, if it satisfies the following condition, then the one-time updated neural network $f\!\mid_{\boldsymbol{\theta}  = \boldsymbol{\theta}^{(t)}}$ is expected to be individually fair.
    
    \begin{equation}
    \label{eq:zerosum}
    \begin{aligned}
        &\mathbb{L}=\left\{(i, j)\left|\quad i>0,~ \forall \bm{x}, \bm{s}_1, \bm{s}_2. \tilde{f}_{(i,j)}(\bm{x},\bm{s}_1, \ldots \boldsymbol{\theta}_{(i,j)}) =\tilde{f}_{(i,j)}(\bm{x},\bm{s}_2, \ldots \boldsymbol{\theta}_{(i,j)}),\right.\right.\\
        &\left.\left.\exists j',\bm{x}, \bm{s}_1, \bm{s}_2.~ \tilde{f}_{(i-1,j')}(\bm{x},\bm{s}_1, \ldots \boldsymbol{\theta}_{(i-1,j')}) \neq\tilde{f}_{(i-1,j')}(\bm{x},\bm{s}_2, \ldots \boldsymbol{\theta}_{(i-1,j')})\right.\right\}.\\
        &\forall (\tilde{\imath}, \tilde{\jmath})\in\mathbb{L}.~\forall i^*\le \tilde{\imath}, j^*. \forall\bm{x}_1\in\mathbb{X}. \quad\operatorname{E} \left[\left.\frac{\partial z_{i^*,j^*}}{\partial \boldsymbol{\theta}_{(i^*,j^*)}}\right|_{\bm{x}=\bm{x}_1, \bm{s}=\textnormal{s}, \boldsymbol{\theta}_{(i^*,j^*)} = \boldsymbol{\theta}_{(i^*,j^*)}^{(t-1)}}\right] = 0
    \end{aligned}
    \end{equation}

Here, $\tilde{f}_{(i,j)}$ represents a subnetwork that takes the same input as the full network and outputs the neuron value $z_{(i, j)}$, as seen in \cref{eq:subnet}.

\end{theorem}

\begin{proof}
    According to \cref{lemma:either}, if the described updates of parameters in this theorem fulfil either condition, then the updated neural network is expected to be individually fair. The condition in \cref{eq:zerosum} states that some neurons' expected updates already satisfy \cref{eq:zero}. The subsequent task is whether the remaining parameters also do. Thus, we seek to prove whether $ \forall (\tilde{\imath}, \tilde{\jmath})\in\mathbb{L}.~\forall i^*> \tilde{\imath}, j^*. \forall\bm{x}_1\in\mathbb{X}, \bm{s}_1, \bm{s}_2\in\mathbb{S},$
    \begin{equation}
    \begin{aligned}
        \left.\frac{\partial z_{i^*,j^*}}{\partial \boldsymbol{\theta}_{(i^*,j^*)}}\right|_{\bm{x}=\bm{x}_1, \bm{s}=\bm{s}_1, \boldsymbol{\theta}_{(i^*,j^*)} = \boldsymbol{\theta}_{(i^*,j^*)}^{(t-1)}} = \left.\frac{\partial z_{i^*,j^*}}{\partial \boldsymbol{\theta}_{(i^*,j^*)}}\right|_{\bm{x}=\bm{x}_1, \bm{s}=\bm{s}_2,\boldsymbol{\theta}_{(i^*,j^*)} = \boldsymbol{\theta}_{(i^*,j^*)}^{(t-1)}~.}
    \end{aligned}
    \end{equation}
    In general, these gradients depend on the partial derivative of the subsequent neurons of each neuron. Since the current neural network state is individually fair, for all neurons on layer $i^*+1$, the gradients would be the same for different sensitive features. Thus, $\partial z_{n,1}/\partial \boldsymbol{\theta}_{(i^*+1,\cdot)}$ is proportional to $\partial z_{i^*+1,\cdot}/\partial \boldsymbol{\theta}_{(i^*+1,\cdot)}$, which depends on $z_{(i^*,j^*)}$. Thus, we get the following expectation.
    \begin{equation}
    \begin{aligned}
    \operatorname{E}&\left[f_{(i^*+1,\cdot)}\left(z_{i^*,j^*},  \ldots, \boldsymbol{\theta}_{(i^*+1,\cdot)}\right)\right]\\
    &= f_{(i^*+1,\cdot)}\left(\operatorname{E}\left[z_{i^*,j^*}\right],  \ldots\right)   + \frac{\nabla^2 f_{(i^*+1,\cdot)}\left(\operatorname{E}\left[z_{i^*,j^*}\right],  \ldots\right)   }{2} \operatorname{Var}\left[z_{i^*,j^*}\right] 
    \end{aligned}
    \end{equation}
    We already observed that $\operatorname{E}\left[z_{i^*,j^*}\right]$ remains the same for different sensitive features. Therefore, for $\operatorname{Var}\left[z_{i^*,j^*}\right] \to 0$, the update of $\boldsymbol{\theta}_{(i^*+1,\cdot)}$ is expected to be 0, and the subsequent neurons can be proven recursively.
    
\end{proof}

The subsequent task would be to find all such neurons in the given individually fair neural network and compute their gradients to determine the value of $\gamma$.   Specifically, we would like to solve Equations~(\ref{eq:solve}) to obtain $\gamma$ before training. In \cref{eq:rrs_plus}, we can see that $P(s|\bm{s})$, the effective probability distribution of $\textnormal{s}$, is a function of $\gamma$. Thus, $\gamma$ is the only variable to solve in Equations~(\ref{eq:solve}), and other quantities like $\frac{\partial z_{i,j}}{\partial \boldsymbol{\theta}_{(i,j)}}$ are known parameters from automatic differentiation.  If we can get at least one $\gamma$ solution, then we can certify this fairness-preserving updating.
\begin{equation}
\label{eq:solve}
\begin{aligned}
    \sum_{s\in\set{\omega_0, \omega_1, \ldots}} \left.\frac{\partial z_{i,j}}{\partial \boldsymbol{\theta}_{(i,j)}}\right|_{ \boldsymbol{\theta}_{(i,j)} = \boldsymbol{\theta}_{(i_1,j_1)}^{(t-1)}} P(s|\bm{s}) &=     \sum_{s\in\set{\omega_0, \omega_1, \ldots}} \left.\frac{\partial z_{i,j}}{\partial \boldsymbol{\theta}_{(i,j)}}\right|_{ \boldsymbol{\theta}_{(i,j)} = \boldsymbol{\theta}_{(i_1,j_1)}^{(t-1)}} \frac{e^{\gamma\cdot\mathbf{1}_{s=\bm{s}}}}{e^\gamma + \abs{\mathbb{S}}-1} &=0\\
    \sum_{s\in\set{\omega_0, \omega_1, \ldots}} \left.\frac{\partial z_{i,j}}{\partial \boldsymbol{\theta}_{(i,j)}}\right|_{ \boldsymbol{\theta}_{(i,j)} = \boldsymbol{\theta}_{(i_2,j_2)}^{(t-1)}} P(s|\bm{s}) &=     \sum_{s\in\set{\omega_0, \omega_1, \ldots}} \left.\frac{\partial z_{i,j}}{\partial \boldsymbol{\theta}_{(i,j)}}\right|_{ \boldsymbol{\theta}_{(i,j)} = \boldsymbol{\theta}_{(i_2,j_2)}^{(t-1)}} \frac{e^{\gamma\cdot\mathbf{1}_{s=\bm{s}}}}{e^\gamma + \abs{\mathbb{S}}-1}&= 0\\
    &&\ldots
\end{aligned}
\end{equation}

Note that we get the parameters of the $\textnormal{s}$ distribution, and we would then be able to get the distribution of each $z_{i,j}$ in \cref{eq:zerosum}. Furthermore, each has a zero expectation and some variance $\sigma_{i,j}$. At actual training time, we will sample the given \textit{training} input with size $\delta$. Therefore, based on Chebyshev's inequality~\cite{chebyshev1867valeurs}, we can derive the following non-asymptotic bound for the probability that the actual update in these parameters is greater than a real number threshold $\tau$.
\begin{equation}
    P\left(z^{(t)}_{i,j} - z^{(t-1)}_{i,j} > \tau\right) \;\le\; \frac{\sigma_{i,j}^2/\delta}{\tau^2}.
\end{equation}
Since $\tau$ is a fixed number given the parameter values, we can always increase the number $\delta$ to keep the outliers at a small amount.

Intuitively, \cref{thm:update} guarantees that fairness is maintained during parameter updates when setting an appropriate prior in the randomised response. Then, \cref{alg:gd} effectively implements this update rule. Together, the network can be updated without compromising its fairness properties.

This method is also scalable and practical, making it suitable for real-world applications. Our approach embeds a guarantee of fairness into the training process, avoiding the need for extensive post hoc analysis such as verification.

In summary, the proposed fairness guaranteeing method starts from an individually fair initialisation (\cref{lemma:init}) and preserves it during training (\cref{lemma:sub,lemma:either,thm:update,alg:gd}). The result is a neural network that is both effective and provably fair, ensuring fairness in sensitive applications.

\begin{proposition}
\label{prop:sound}
    Let $A$ denote our training in \cref{alg:gd} such that $A(f\!\mid_{\boldsymbol{\theta}})$ would be the updated state of $f\!\mid_{\boldsymbol{\theta}}$. Then, when training an individually fair neural network with $A$, the result updated neural network is guaranteed to be also individually fair. Formally, 
    \begin{equation}
        (F(f\!\mid_{\boldsymbol{\theta}}, \mathbb{X}, \mathbb{S})=1) \to (F(A(f\!\mid_{\boldsymbol{\theta}}), \mathbb{X}, \mathbb{S}) = 1).
    \end{equation}   
\end{proposition}

\section{Experiments}
\label{sec:exp}

In the following, we conduct experiments to quantitatively evaluate the proposed method. We focus on empirical effectiveness and efficiency. Specifically, we attempt to answer the following research questions. (1) Would our model initialisations be provably fair? (2) Empirically, would our trained classifiers exhibit fair individual fairness? (3) How good is the utility of our trained classifier? (4) Does the training process maintain efficiency while ensuring individual fairness?

We first describe the experiment setup, including the benchmarking datasets, baselines, and implementation details. Then, we report the results, give our discussions, and answer each of these research questions.

\subsection{Experimental Setup}

\paragraph{Benchmarks}
The experiment includes six datasets and classification tasks that are commonly studied~\cite{zhang2020white,zheng2022neuronfair} where individual fairness is a concern. Specifically, Bank Marketing (Bank)~\cite{moro2014data} has 45,000 records to predict whether a client will subscribe to a term deposit~\cite{asuncion2007uci}. German Credit (Credit) has 1,000 samples to assess credit risk~\cite{statlog_(german_credit_data)_144}. Census Income (Adult)~\cite{adult_2} has 32,561 records to predict if one earns over \$50,000. Compas~\cite{angwin2022machine} has 6,172 samples to predict recidivism risk. Heritage Health Prize (Health)~\cite{health_prize} has 218,415 historical claims data that can predict patient mortality. Lastly, Law School Admissions Bar Passage (Law)~\cite{wightman_2017} tracks 22,408 records from 1991 through 1997 to predict
if a student passes the bar. For each dataset, a random set containing 20\% of samples is used for testing, and the remaining serves for training and validation. The respective sensitive features are race and gender for Census, age for Bank, age and gender for Credit, race for Compas, and gender for Health and Law.

\paragraph{Baselines}

We compare the proposed method with existing methods that provide some guarantee on individual fairness. The baselines include prior works on training methods that improve and certify local (LCIFR~\cite{ruoss2020learning}) or distributional (U-DIF~\cite{wicker2023certification}) individual fairness. In detail, LCIFR learns fair representations with a loss function regulated by the similarity constraint, and certifies fairness afterwards. Similarly, U-DIF adds a convex approximation of fairness constraints to the training loss and provides post-hoc distributional guarantees. We also include MILP~\cite{benussi2022individual}, which indirectly guarantees individual fairness by certifying a maximum possible global outcome change. Note that these guarantees are fundamentally different from ours.    Other fairness-focused training methods \cite{hardt2016equality} exist. However, they are not included in the comparison because they do not provide (and also do not aim to provide) a formal fairness guarantee. Besides, we include empirical risk minimisation (ERM) that optimises towards lower cross entropy between label and system prediction~\cite{vapnik1999nature}. ERM is a standard method focused on accuracy.

\paragraph{Implementation}

For all classification tasks, we use a six-layer fully connected neural network, following prior fairness studies~\cite{zhang2020white,zheng2022neuronfair}. This architecture is composed of 5 hidden layers that are activated with ReLU, and one non-activated output layer. The number of neurons at each hidden layer is [64, 32, 16, 8, 4]. The models are initialised with zero mean in weights and small variance to break symmetry. During training, the learning rate is set to 0.01 and the batch size is set to 64. Our experiments are conducted on a GPU server using a Standard\_NC8as\_T4\_v3 Azure cloud instance, equipped with a single NVIDIA Tesla T4 GPU with 16 GB of VRAM, 8 vCPUs, based on AMD EPYC 7V12 (Rome) processors, and 56 GiB of RAM.

\subsection{Result}

In the following, we report the experiment results regarding the answers to each research question (RQ). In brief, we present \cref{tab:verify}, \cref{tab:acc}, and \cref{fig:iter}. \cref{tab:verify} is mainly used to answer RQ1. \cref{tab:acc} contains the result of empirical individual fairness and accuracy, supporting RQ2 and RQ3. \cref{fig:iter} illustrates a dynamic statistics change along training and mainly corresponds to RQ4. 

\paragraph{Result for RQ1: Would our model initialisations be provably fair?}

\begin{table}[t]
\small
      \centering
      \caption{The individual fairness verification results for neural networks by Fairfy~\cite{biswas2023fairify}.}
    \begin{tabular}{@{}lcrrr@{}}
    \toprule
    \textbf{Dataset} & \textbf{Verification} & \textbf{\#Partition} & \textbf{Time (minutes)} & \textbf{Accuracy (\%)} \\ \midrule
    Census & Provably fair & 16,000    & 38.1  & 75.4 \\
    Bank   & Provably fair & 510    & 4.2 & 71.0    \\
    Credit & Provably fair & 201 & 2.3  & 75.55 \\
    Compas & Provably fair & 14,000 & 12.0  & 55.6  \\
    Health & Provably fair & 6,400 & 14.9 & 68.0    \\
    Law    & Provably fair & 13,394 & 14.5 & 73.7  \\ \bottomrule
    \end{tabular}
      \label{tab:verify}
\end{table}

\cref{lemma:init} establishes the existence of provable individually fair classifier initialisations, and this research question asks whether we can actually attain them empirically. To this purpose, we feed each initialised model to Fairfy~\cite{biswas2023fairify}, an existing sound fairness verification program. The verification results presented in \cref{tab:verify} indicate that across all datasets the initialised models consistently achieve provable individual fairness. Note that a successful verification result indicates that these initial states of classifiers satisfy individual fairness constraints for the entire input space.

While the above result well answers RQ1, we believe that additional details could provide insights. For instance, how easily can a classifier (whether initialised or already trained) be verified by Fairfy? The work of Fairfy~\cite{biswas2023fairify} highlights an intriguing phenomenon. Out of 25 frequently used classifiers in fairness testing~\cite{zhang2020white,zheng2022neuronfair}, none of them have been verified (fair) by Fairfy. By contrast, our six initialised classifiers are verified. This might suggest that an already trained classifier tends to be more difficult to verify, but verifying initialised models (\emph{e.g.}, shown in \cref{lemma:init}) could likely be a smooth task. 

We then look into the verification statistics in detail. For instance, in the Census task, a model initialisation achieves a test accuracy of 75.5\%. When fed into Fairfy~\cite{biswas2023fairify}, the entire input space is divided into 16,000 partitions, and it takes 38.1 minutes to complete and obtain a proof. For a reference classifier AC7~\cite{zhang2020white,biswas2023fairify} (with accuracy 84.85\%), it takes 178.1 minutes to verify only 21 of 16,000 partitions, among which 15 are falsified and 6 are undecided. A similar trend is observed in other datasets. This suggests that the model initialisations allow fast and effective verification, despite lower accuracy. Overall, the results demonstrate that the initialised models successfully achieve provable individual fairness across all datasets. This outcome is critical to support the individual fairness of our trained classifiers.

\paragraph{Result for RQ2: Empirically, would our trained classifiers exhibit individual fairness?}

\begin{table}[t]

\small
\centering
\caption{Empirical individual fairness and accuracy across datasets and methods. For each dataset, we present the results regarding the protected attributes.
\label{tab:acc}
}
\begin{tabular}{@{}l|crlrl|ccccc@{}}
\toprule
\textbf{}                       & \multicolumn{5}{c|}{\textbf{Empirical individual fairness} (\%)}                  & \multicolumn{5}{c}{\textbf{Accuracy} (\%)}         \\ \midrule
\textbf{Dataset (attr.)} & \textbf{ERM} & \textbf{LCIFR} & \textbf{MILP} & \textbf{U-DIF} & \textbf{Ours} & \textbf{ERM} & \textbf{LCIFR} & \textbf{MILP} & \textbf{U-DIF} & \textbf{Ours} \\\midrule
Census (age)    & 96.80 & 100.00 & 92.83 & 100.00 & 100.00 & 83.30 & 83.10 & 76.20 & 82.12 & 83.23 \\
Census (race)   & 96.64 & 100.00 & 99.22 & 100.00 & 100.00 & 83.30 & 83.22 & 78.18 & 83.15 & 83.24 \\
Census (gender) & 96.36 & 100.00 & 94.68 & 99.72  & 100.00 & 83.30 & 83.10 & 83.10 & 83.17 & 83.23 \\
Bank   (age)    & 84.19 & 99.94  & 96.93 & 100.00 & 100.00 & 89.00 & 88.50 & 88.04 & 87.93  & 88.51 \\
Credit (age)    & 88.50 & 100.00 & 83.50 & 100.00 & 100.00 & 75.50 & 74.50 & 75.00 & 70.50 & 75.50 \\
Credit (gender) & 87.00 & 100.00 & 90.00 & 100.00 & 100.00 & 75.50 & 75.00 & 75.00 & 70.49 & 75.00 \\
Compas (race)   & 90.11 & 100.00 & 95.71 & 100.00 & 100.00 & 72.53 & 66.37 & 55.95 & 55.37 & 67.75 \\
Health (gender) & 64.10 & 99.10  & 96.61 & 100.00 & 100.00 & 80.70 & 80.90 & 70.07 & 68.93 & 80.70 \\
Law    (gender) & 25.60 & 51.10  & 98.62 & 90.36 & 100.00 & 84.41 & 84.40 & 75.74 & 83.96 & 84.40\\ \bottomrule
\end{tabular}
\end{table}

\cref{prop:sound} establishes that if our initialisation is individually fair, then our trained model (through \cref{alg:gd}) is also individually fair. Furthermore, sound verification (as shown in RQ1) entails the individual fairness of our initialisations. Therefore, our trained models are already proven individually fair.

While the individual fairness of our trained models is formally proven, one may still be interested in the actual behaviour of these trained models. To this end, we evaluate the empirical individual fairness, \emph{i.e.}, for each input in the test dataset, \emph{rather than the entire input space}, we check if changing the sensitive feature changes the classifier outcome. Our empirical test allows comparison with other training methods, regardless of whether a formal fairness guarantee is global or local.

The empirical test result regarding individual fairness is shown in \cref{tab:acc}, under the ``Empirical individual fairness~(\%)'' tab. Results show that our trained models are fair regarding all inputs across all datasets, which is not surprising because these models have formally proved to satisfy global individual fairness. 

This empirical performance also surpasses all four baselines. Among the baselines, ERM almost always achieves the lowest empirical fairness, with an average of 80.42\% across all datasets. The largest fairness gap is observed in Law School (25.60\%), where more than 74\% of instances violate individual fairness, while the highest fairness under ERM is in the Census (96.80\%) result.

LCIFR indeed improves fairness, reaching an average of 94.56\% across datasets, with a 17.6\% relative increase over ERM. Besides, in three datasets (Census, Credit, and Compas), LCIFR has 100\% empirical fairness. Note that LCIFR guarantees only local individual fairness rather than global individual fairness. Specifically, LCIFR retains discriminatory samples in Health (99.10\%), Bank (99.94\%), and Law School (51.10\%), meaning it explicitly fails to provide empirical fairness in one-third of the datasets. U-DIF achieves an average of 25.6\% higher empirical fairness than LCIFR on Bank, Health, and Law, but falls short on Census (gender). MILP achieves lower empirical fairness than LCIFR or U-DIF, by 0.3\% or 4.8\%, respectively, which is likely because it is not directly optimised for individual fairness (as in \cref{def:if}).

Since individual fairness is a matter for every single person, we also study the absolute counts of discriminatory samples. \textcolor{black}{While LCIFR achieves 99.10\% fairness in Health and U-DIF achieves 99.72\% in Census (gender), this still leads to 393 (Health) or 18 (Census) individuals being treated unfairly, a non-negligible number.} This reveals that even almost perfect fairness percentages can still leave individuals subject to unfair treatment. \textcolor{black}{Additionally, in Law, fairness improves from 25.60\% (ERM) to 51.10\% (LCIFR) or 90.36\% (U-DIF), meaning LCIFR is twice as fair as ERM, and U-DIF is 3.5 times as fair as ERM. However, this still implies that 2,191 individuals (LCIFR) or 423 individuals (U-DIF) remain treated unfairly.}  In summary, our method completely eliminates all discriminatory samples and empirically outperforms all baselines.

Note that while this empirical fairness test does not further add to the strictness of our formal proof, it does highlight the consistency and effectiveness of our method in reducing the number of discriminatory samples across all datasets.

\paragraph{Result for RQ3: How good is the utility of our trained classifier?}

\cref{tab:acc} (under the ``Accuracy (\%)'' tab) presents a comparative analysis of the vanilla accuracy across various datasets and methods (ERM, LCIFR, MILP, U-DIF, and ours), highlighting their utility in practice. ERM reports an average accuracy of 83.30\% for the Census dataset, which remains consistent across all protected attributes (age, race, and gender). LCIFR drops slightly to 83.15\%$\pm$0.06, depending on the selected sensitive attribute. The proposed method offers a modest increase from LCIFR to 83.23\%. These small differences suggest that these three methods (ERM, LCIFR, and ours) perform similarly in minimising errors for this dataset. 

In the bank dataset, the vanilla accuracy is higher, with ERM reporting 89.00\%, LCIFR slightly dropping to 88.50\%, and ours being in between. The variation across methods is minimal, with differences of less than 5\%, indicating that the proposed method maintains a stable performance in this dataset despite addressing discriminatory outcomes. For smaller datasets like Credit and Compas, incorrect classifications become more common across all methods, with almost a quarter of all testing samples. Yet, the accuracy gap among these training methods is small ($<$0.5\%) for Credit and relatively large ($\pm$3\%) for Compas. This could likely be attributed to the small dataset size~\cite{dietterich1995overfitting}.

Besides, our method achieves, on average, a 4.2\% higher utility compared to the MILP utility or the U-DIF utility. U-DIF achieves similar utility (to ERM, LCIFR, and ours) on Census, Bank, and Law, although still 0.5\% lower in average. Yet, its utility on Credit, Compas, and Health is around 12\% lower on average. MILP achieves similar utility (to ERM, LCIFR, and ours) on Bank and Credit, but lower on other datasets.

Overall, the average accuracy across all datasets and methods reveals a consistent trend. ERM reports an average accuracy of 80.84\%, which slightly decreased by about 1.2\% with LCIFR and is marginally reduced by about 0.83\% with the proposed method. Our approach consistently performs closer to ERM, offering a modest improvement over LCIFR. From the perspective of protected attributes, the variation in accuracy under the compared methods is consistently small, typically ranging between 0.2\% $\sim$1\% across all datasets and attributes.

\begin{figure}[t]
    \centering
    \begin{subfigure}[t]{0.33\linewidth}
    \centering
    \includegraphics[width=\linewidth]{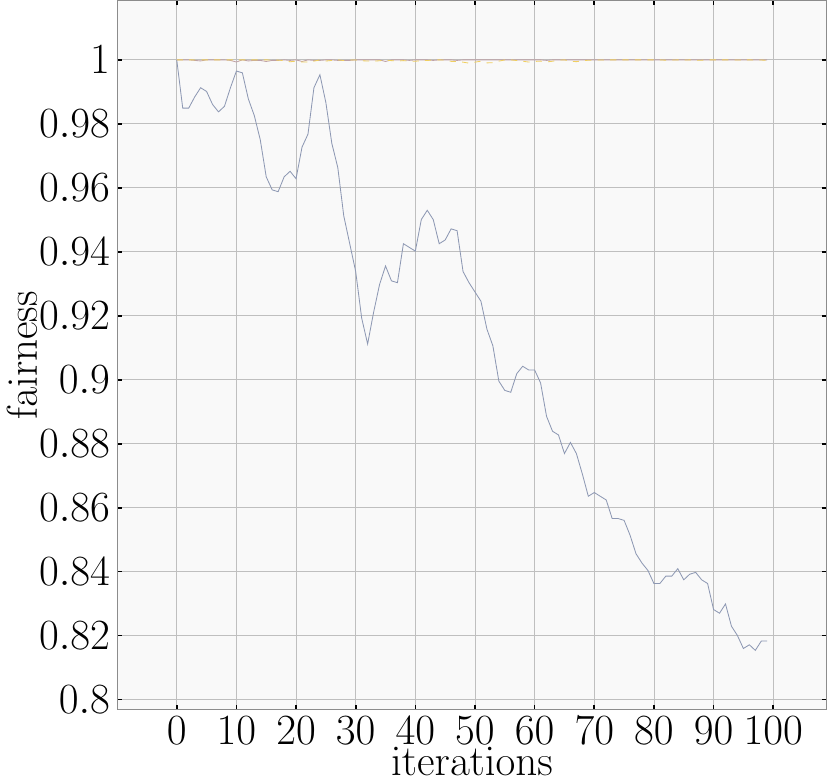}
    \caption{Individual fairness\label{fig:if}}
    \end{subfigure}%
    \begin{subfigure}[t]{0.33\linewidth}
    \centering
    \includegraphics[width=\linewidth]{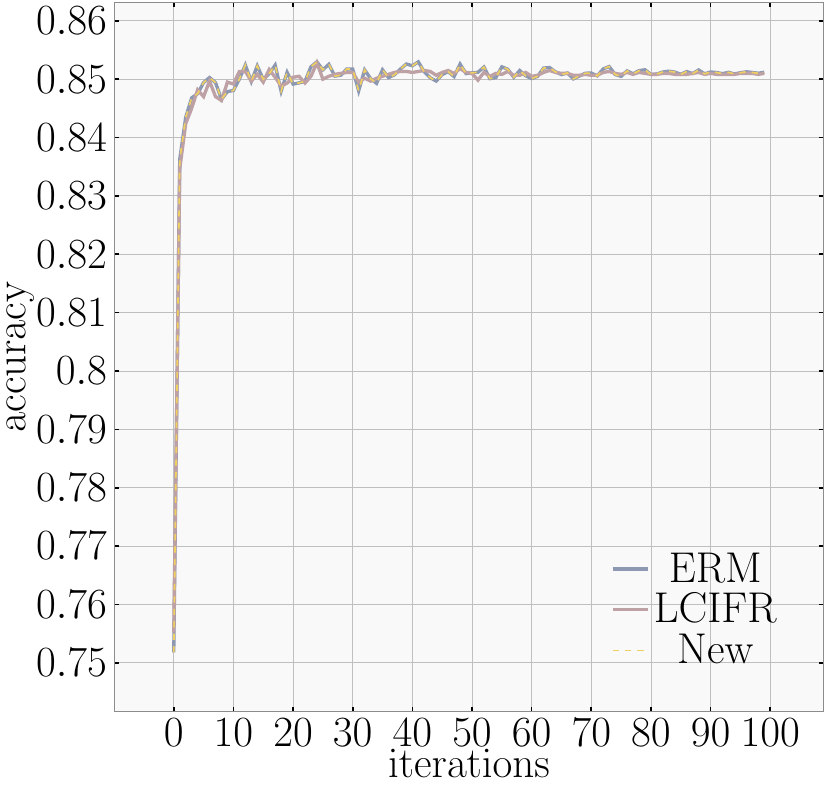}
    \caption{Accuracy\label{fig:acc}}
    \end{subfigure}%
    \begin{subfigure}[t]{0.33\linewidth}
    \centering
    \includegraphics[width=\linewidth]{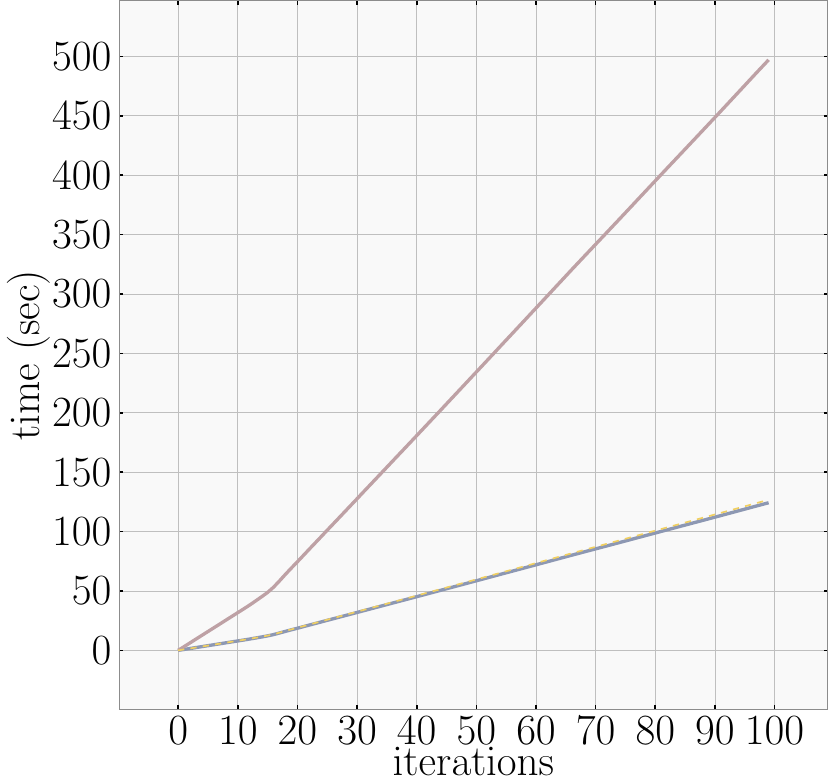}
    \caption{Running time\label{fig:efficiency}}
    \end{subfigure}%
    \caption{As iteration increases, we plot the empirical individual fairness, accuracy, and running time of each classifier in the Census dataset~\cite{adult_2}.
    }
    \label{fig:iter}
\end{figure}

To understand the utility difference between the proposed method and existing methods like LCIFR, we analyse our misclassified examples, especially the cases where other methods make correct predictions. For instance, in Census, there are 98 such cases out of 6,512 testing examples. It is observed that these 98 cases all have positive class labels, indicating that the analysed cases tend to be false negatives. This may be a result of label distribution or applying fairness constraints, leading the model to be more careful when assigning positive labels.

In summary, the proposed method reduces the accuracy slightly compared to ERM. However, it generally aligns closer to ERM, offering better performance than LCIFR in minimising errors while significantly reducing discriminatory samples. These results suggest that the trade-off between enhancing individual fairness and maintaining accuracy can be small, \emph{i.e.}, protecting sensitive attributes does not ruin the utility.

\paragraph{Result for RQ4: Does the training process maintain efficiency while ensuring individual fairness?}

To answer the fourth research question, \emph{i.e.}, whether pursuing individual fairness reduces the computational efficiency of training a classifier, we need to compare the results shown in \cref{fig:iter}. Specifically, \cref{fig:efficiency} most directly compares the time of various training. While ERM is still considered the simplest and most efficient among the three, the proposed training only has a modest time growth on top of that. For 100 epochs on the Census dataset, the proposed training takes around 122 seconds to complete, which is only 2.04\% more than that taken by ERM. By contrast, existing training like LCIFR takes around 497 seconds, which is more than 4.15 times the time taken by ERM. In summary, the new method is about four times faster than the existing LCIFR method. This is likely because the proposed training is end-to-end training that does not need to train a neural model/classifier part by part. As observed in \cref{fig:efficiency}, the training time per iteration is almost a constant for each training, \emph{i.e.}, the early epochs would not take a considerably shorter period compared with the later iterations. Thus, we can approximately compute the training speed as the number of iterations divided by the total time. In detail, the training speeds for ERM, LCIFR, and ours on Census are 0.83 steps/sec, 0.82 steps/sec, and 0.20 steps/sec, and the ratio is 1:0.99:0.24. Applying the same analysis, we found that the speed ratio is  1:0.99:0.24, with $\pm$0.001 fluctuation range, for all other datasets.

Besides comparing the absolute training directly as seen in \cref{fig:efficiency}, computation efficiency may also be observed according to the convergence speed. To the accuracy end as demonstrated in \cref{fig:acc}, all three training can achieve 95\% of their respective converged accuracy, \emph{i.e.}, the accuracy fluctuation is within $\pm$5\% of the converged accuracy. As for (empirical) individual fairness performance illustrated in \cref{fig:if}, LCIFR and the proposed training both maintain high empirical fairness from the start, and both are constantly higher than the individual fairness of ERM. Thus, the computational efficiency (of the proposed method) is not harmed from this perspective. 

Therefore, we can answer this research question that the key efficiency difference still lies in the training time per iteration, where the proposed method is around four times faster than the existing method with individual fairness guarantees.

\paragraph{Ablation: How does model configuration influence performance?}
\begin{figure}
    \centering
    \includegraphics[width=0.7\linewidth]{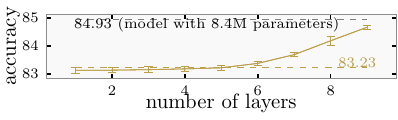}
    \caption{Impact of network configurations on utility. The dataset here is Census~\cite{adult_2}.\label{fig:config}}
    \label{fig:enter-label}
\end{figure}
To explore the applicability of the proposed method to other neural network configurations, we extend our experiments to shallower or deeper neural networks in literature~\cite{ruoss2020learning,zheng2022neuronfair}. Specifically, we evaluate networks ranging from a single hidden layer with 16 neurons to a deeper nine-layer network with the configuration [256, 256, 64, 64, 32, 32, 16, 8, 4]. Across these configurations, we observe that empirical fairness is consistently maintained. The corresponding utility results are presented in \cref{fig:config}. As shown, the shallower architecture results in a slight decrease in accuracy (by 0.09\%), while the deeper network achieves an absolute accuracy improvement of up to 1.2\%. Additionally, we evaluate a large neural network with approximately 8.4 million parameters, using the configuration [2048, 2048, 2048, 4]~\cite{benussi2022individual}, which achieves an accuracy of 84.93\%. This aligns with the intuition that appropriately scaled larger models could improve test accuracy. More importantly, it shows that the proposed method can scale well.

\section{Related Works}
\label{sec:related}

\textcolor{black}{The problem of ensuring satisfactory individual fairness in machine learning has drawn a lot of attention~\cite{bender2021dangers}.} This problem is typically addressed through two main perspectives, \emph{i.e.}, fairness optimisation and fairness verification. Each category addresses a part of this problem. Verification-based methods aim to provide formal guarantees by detecting fairness violations in a neural network across its input space~\cite{biswas2023fairify,kim2025fairquant,bastani2019probabilistic,li2023certifying}. Approaches such as FairQuant~\cite{kim2025fairquant} leverage symbolic analysis of deep neural networks, combined with iterative refinement, to identify unfair outcomes. These methods provide frameworks for identifying unfair outcomes but still leave an unavoidable portion undecided. Besides, the verification methods cannot provide fairness improvement.

Fairness-aware training methods address discrimination by incorporating fairness objectives into the learning process. Techniques such as adversarial sampling~\cite{zhang2021automatic,xu2021robust} and fairness-optimised representations~\cite{lahoti2019operationalizing,sarridis2024flac} adjust model training to mitigate bias. \textcolor{black}{For example, prior work~\cite{zhou-etal-2023-causal} aims to achieve fairness by modifying loss functions or introducing regularisation terms to balance fairness and utility. Alternatively, there are approaches to mitigating unfairness through post-processing neural network repair work like DICE~\cite{monjezi2023information} and NeuFair~\cite{dasu2024neufair}.} However, they often lack formal guarantees of fairness and may reduce model performance~\cite{wang2021understanding}.

Recently, methods focusing on certifiable fairness have gained attention, aiming to bridge the gap between theoretical guarantees and practical applicability. Works like certifiable individually fair representations (LCIFR)~\cite{peychev2022latent} map similar individuals to close latent representations and leverage the certified robustness method to achieve certified fairness. Yet, such work focuses on certifying a local individual fairness property, which is different from the global perspective in this study.  Similarly, \citet{wicker2023certification} apply a convex approximation of fairness constraints to guarantee individual fairness for distributions in a specific Wasserstein ball (U-DIF). This provides a distributional guarantee, covering a larger space than the local guarantee, but still missing a noticeable portion of the global space. \citet{benussi2022individual} study individual fairness guarantee in a global perspective (MILP). Yet, MILP focuses on a maximum possible outcome change, rather than directly answering whether fairness is satisfied. Thus, these guarantees are fundamentally different from ours.

On the other hand, certified group fairness~\cite{jin2022input,kang2022certifying} studies whether the balance is achieved among demographic groups. This is a different type of problem from the one in this work. To the best knowledge, our method is the first to achieve global individual fairness with formal guarantees.

\section{Conclusion}
\label{sec:conclusion}

This study tackles the problem of certifying the individual fairness of a well-performing neural classifier. Individual fairness has recently been acknowledged as a highly desirable property of machine learning algorithms in social and ethics-related classification tasks like income decisions. While individual fairness may be empirically improved, it is hard to guarantee this property, especially when the utility is high. To address this problem, we propose sound fairness-preserving training based on randomised responses of sensitive features where no bias is injected in this process. This training is particularly helpful given that it is applied to a provably fair initialisation. We formally show the existence of fair initiation and establish that no bias can be injected when an appropriate prior is set in the randomised response of sensitive features. The empirical fairness testing on six benchmark tasks also aligns with our theoretical results. Moreover, the proposed method certifies individual fairness through training and avoids computationally expensive post hoc fairness verification, which increases efficiency. To summarise, the proposed method well addresses our problem, without compromising utility or efficiency.

\section*{Data-Availability Statement}
Our code, data, and results are available at \url{https://github.com/cat-claws/correct-by-construction}.

\begin{acks}
This research is partially supported by the Lee Kuan Yew Fellowship Fund awarded to Professor SUN Jun.
\end{acks}

\bibliographystyle{ACM-Reference-Format}

\appendix

\end{document}